\documentclass[nohyperref]{article}

\usepackage{microtype}
\usepackage{graphicx}
\usepackage{subfigure}
\usepackage{booktabs} % for professional tables

\usepackage{hyperref}

\usepackage[accepted]{icml2022}

\usepackage{amsmath}
\usepackage{amssymb}
\usepackage{mathtools}
\usepackage{amsthm}

\usepackage[capitalize,noabbrev]{cleveref}

\theoremstyle{plain}
\newtheorem{theorem}{Theorem}[section]

\newtheorem{lemma}[theorem]{Lemma}

\theoremstyle{definition}

\theoremstyle{remark}

\newtheorem{conjecture}[theorem]{Conjecture}

\usepackage[textsize=tiny]{todonotes}

\icmltitlerunning{Invariant Layers for Graphs with Nodes of Different Types}

\begin{document}

\twocolumn[
\icmltitle{Invariant Layers for Graphs with Nodes of Different Types}

% Invariant Layers for Graphs with Nodes of Different Types


\icmlsetsymbol{equal}{*}

\begin{icmlauthorlist}
\icmlauthor{Dmitry Rybin}{yyy}
\icmlauthor{Ruoyu Sun}{yyy}
\icmlauthor{Zhi-Quan Luo}{sch}
\end{icmlauthorlist}

\icmlaffiliation{yyy}{School of Data Science, The Chinese University of Hong Kong, Shenzhen}
\icmlaffiliation{sch}{Shenzhen Research Institute of Big Data, The Chinese University of Hong Kong, Shenzhen}

\icmlcorrespondingauthor{Dmitry Rybin}{dmitryrybin@link.cuhk.edu.cn}

\icmlkeywords{graph neural network, expressive power, invariant, tensor, permutation}

\vskip 0.3in
]



\printAffiliationsAndNotice{}  % leave blank if no need to mention equal contribution

\begin{abstract}
Neural networks that satisfy invariance with respect to input permutations have been widely studied in machine learning literature. However, in many applications, only a subset of all input permutations is of interest. For heterogeneous graph data, one can focus on permutations that preserve node types. We fully characterize linear layers invariant to such permutations. We verify experimentally that implementing these layers in graph neural network architectures allows learning important node interactions more effectively than existing techniques. We show that the dimension of space of these layers is given by a generalization of Bell numbers, extending the work \cite{Maron2018}. We further narrow the invariant network design space by addressing a question about the sizes of tensor layers necessary for function approximation on graph data. Our findings suggest that function approximation on a graph with $n$ nodes can be done with tensors of sizes $\leq n$, which is tighter than the best-known bound $\leq n(n-1)/2$. For $d \times d$ image data with translation symmetry, our methods give a tight upper bound $2d - 1$ (instead of $d^{4}$) on sizes of invariant tensor generators via a surprising connection to Davenport constants.
\end{abstract}

\section{Introduction}
\label{introduction}
The study of invariant and equivariant neural networks has been gaining popularity in recent years. Many fundamental properties, such as universal approximation theorems \cite{Maron2019}, \cite{yarotsky, ravanbakhsh2020}, have been proved. The design of expressive invariant layers remains an important direction in Deep Learning \cite{Hartford2018, kondor2018}.

Permutation invariant networks are an important special case. In these networks, the symmetry group of the input data is given by all possible permutations of input coordinates. In particular, such symmetry appears in the use of graph neural networks \cite{kipf2016}, where the invariance comes from the permutation of nodes. This symmetry is crucial for architecture design in graph neural networks and the study of their expressive power and universal approximation properties \cite{Chen2019, frasca2022, garg2020, Huang2022, xu2018, bevilacqua2022, qian2022}. However, node permutations in homogeneous and heterogeneous graphs have certain differences that received little attention in the literature.

In applications with heterogeneous graphs, the problem background often requires certain groups of nodes to have significantly different properties or represent objects of different nature. Examples of graph applications with many node types can be found in recommendation systems \cite{wu2020}, chemistry \cite{reiser2022}, and Learn-to-Optimize \cite{gasse2019}. Many Graph Neural Network architectures attempt to capture the relations between nodes of different types. Despite many theoretical guarantees \cite{glass}, some simple features are still hard to learn in practice with existing layers, see Section \ref{experiments} for experimental evidence. In this paper, we aim to fix this gap by characterizing all invariant linear layers for permutations preserving node types, hence extending the work \cite{Maron2018} to heterogeneous graphs.

\setlength{\tabcolsep}{0.3em}
\begin{table}[H]
\centering
\begin{tabular}{|c|c|c|c|}
\hline
Tensor sizes & 1 & 2 & 3 \\ \hline
\cite{Maron2018} & 1 & 2 & 5 \\ \hline
\textbf{This work} & $m$ & $m^{2} + m$ & $m^{3} + 3m^{2} + m$ \\ \hline
\end{tabular}
\caption{The dimension of space of invariant linear layers in graphs with $m$ node types. Setting $m = 1$ recovers results from \cite{Maron2018}.}
\label{comparison}
\end{table}
\setlength{\tabcolsep}{0.5em}

\setlength{\tabcolsep}{0.1em}
\begin{table}[H]
    \centering
    \begin{tabular}{|c|c|c|}
        \hline
        Ref. & \cite{Maron2018} & Theorem \ref{labeldim} \\ \hline
         & Classification of &  Classification of \\ 
        Results & $S_{n}$-invariant  & $S_{n_1} \times ... \times S_{n_m}$- \\ 
         & linear layers & -invariant linear layers \\
         & $\mathbb{R}^{n^{k}} \to \mathbb{R}$. & $\mathbb{R}^{n^{k}} \to \mathbb{R}$. \\ \hline
    \end{tabular}
    
    \caption{Comparison of theoretical contribution to the past work.}
    \label{tab:comparison}
\end{table}
\setlength{\tabcolsep}{0.5em}

In Section \ref{classification}, Theorem \ref{labeldim}, we provide a complete characterization of linear layers $\mathbb{R}^{n^{k}} \to \mathbb{R}$ with $k$-tensor input, invariant to permutations from $S_{n_1} \times S_{n_2} \times ... \times S_{n_m}$, where $n_1 + ... + n_m = n$ is the number of nodes and $m$ is the number of different types of nodes. The dimension of the complete space of these layers is given by a generalization of Bell numbers, as can be seen in Table \ref{tab:comparison}. The fact that the structure of invariant layers depends only on the number of node types $m$ and tensor size $k$, and does not depend on $n$, is crucial for the re-use of these layers in graph neural networks. It follows that these layers can be directly applied to any input graph with $m$ types of nodes, independent of the number of nodes. We provide an explicit orthogonal basis and implementation description for these layers (Theorem \ref{explicitform}).

A complete characterization of invariant/equivariant tensor layers defines a design space for invariant neural networks. It was discovered \cite{Maron2019, ravanbakhsh2020, keriven2019, OpenProblemsMaron2019} that higher-order tensors are necessary for function approximation with invariant/equivariant neural networks. However, explicit bounds on required tensor sizes are needed.

The work \cite{Maron2019} showed that $n(n-1)/2$-tensors are sufficient for function approximation on graph data with $n$ nodes. Lowering this bound is an open question posed in \cite{Maron2019} and \cite{keriven2019}. In Section \ref{maximal_tensor}, we provide several theorems and conjectures suggesting a bound below $n$. Furthermore, we show how the structure of orbits of the graph automorphism group determines how tensor sizes can be lowered. Our findings are formulated as conjectures A and B. We summarize the implications of each conjecture in Table \ref{tab:implications}. The mathematical formulation of the conjectures is provided in Section \ref{maximal_tensor}. Some applications, such as image graphs in computer vision, provide special cases that can be analyzed fully. We prove that translation-invariant function approximation on $d \times d$ images can be done with tensors of size $\leq 2d - 1$ (Theorem \ref{cnn}). The proof makes a surprising connection between translation-invariant tensor layers and Davenport constants \cite{davenport}.

\begin{table}[H]
    \centering
    \begin{tabular}{|c|c|}
        \hline
        Result & Description \\ \hline
         & A decrease of tensor sizes in CNN \\
        Theorem \ref{cnn} &  for translation-invariant function \\
         &  approximation from $d^4$ to $2d - 1$ \\ \hline
         & A decrease of required tensor sizes  \\
        Conjecture A &  from $n(n-1)/2$ to $n$ for function  \\
         &  approximation on graphs with $n$ nodes. \\ \hline
        Conjecture B & Graph instance-dependent bound  \\ 
         & on required tensor sizes. \\
        \hline
    \end{tabular}
    \caption{Implications of the proposed conjectures and theorems.}
    \label{tab:implications}
\end{table}

Finally, we discuss the differences between our work and prior results. Due to practical importance, treating different types of nodes has been approached in many applications. For bipartite graphs, half-GNN \cite{gasse2019} and EvenNet were proposed \cite{lei2022}. Layers in half-GNN can be viewed as a special case of ours. For extracting properties of a subset of vertices, subgraph structure extraction (Sub-GNN) \cite{alsentzer2020}, labeling tricks (GLASS) \cite{wang2022}, and other approaches \cite{sun2021, you2021, kexin2020} were used. Subgraph data pooling in Sub-GNN and GLASS is an example of a layer invariant only to permutations of nodes within a subgraph and hence is a special case of our layers. Tensor sizes for function approximation with convolutional networks were analyzed in \cite{yarotsky}. As pointed out in \cite{yarotsky}, finding a small explicit set of generators of translation-invariant tensors is not trivial. Therefore the work \cite{yarotsky} took an alternative approach of averaging the outputs over the whole symmetry group. We bypassed the difficulty by making a change of basis and noting a connection with zero-sum sequences in groups (proof of Theorem \ref{cnn}).

\section{Preliminaries}
\label{preliminaries}

A function $f: \mathbb{R}^{n} \to \mathbb{R}$ is invariant to a permutation $P$ if for any input $x \in \mathbb{R}^{n}$ we have
$$f(x_1, x_2, ..., x_n) = f(x_{P(1)}, x_{P(2)}, ..., x_{P(n)}).$$
For linear functions $f(x) = a_1x_1 + ... + a_n x_n$ this condition is equivalent to a fixed point equations \cite{Maron2018}, which can be solved by analyzing orbits of indices,
$$a_1x_1 + a_2x_2 + ... + a_nx_n = a_1x_{P(1)} + a_2x_{P(2)} + ... + a_nx_{P(n)}.$$
For tensor input data, $x \in \mathbb{R}^{n^{k}}$, coordinates are indexed by $k$-tuples $x_{i_1, ..., i_k}$, where $i_1, ..., i_k \in \{1, 2, ..., n\}$. Permutation $P$ now acts on all elements in a tuple, simultaneously permuting indices in all $k$ axes of a $k$-tensor. For a linear map $f: \mathbb{R}^{n^{k}} \to \mathbb{R}$, invariance to permutation $P$ is equivalent to a fixed point equation
$$\sum a_{i_1, ..., i_k} x_{i_1, ..., i_k} = \sum a_{i_1, ..., i_k} x_{P(i_1), ..., P(i_k)}.$$
While for a linear map between $k$-tensors and $d$-tensors $f: \mathbb{R}^{n^{k}} \to \mathbb{R}^{n^{d}}$, equivariance to permutation $P$ is stated as
$$f(Px) = Pf(x).$$

In the fundamental work \cite{Maron2018}, linear maps $\mathbb{R}^{n^{k}} \to \mathbb{R}$ invariant to all possible $n!$ permutations $P$ were explicitly classified. The dimension of space of these maps was shown to be given by Bell numbers $B(k)$. This classification is important since it provides a complete design space for invariant and equivariant neural networks.

Recall that a standard model for the invariant neural network is a function 
$$F : \mathbb{R}^{n} \rightarrow \mathbb{R}$$ 
defined as
\begin{equation}\label{invariant_network}
F(x) = M \circ h \circ L_d \circ \sigma \circ \cdots \circ \sigma \circ L_1,  
\end{equation}
where $L_i: \mathbb{R}^{n^{k_i} \times a_i} \to \mathbb{R}^{n^{k_{i + 1}} \times a_{i+1}}$ are linear equivariant layers (quantity $a_i$ is the number of filters or channels in layer $i$), $\sigma$ is an activation function (such as ReLU or sigmoid), $h: \mathbb{R}^{n^{k_{d + 1}} \times a_{d + 1}} \to \mathbb{R}^{m}$ is an invariant layer, and $M$ is a multi-layer perceptron. Here layers $L_i$ are equivariant to a predefined set of permutations $P$. While layer $h$ is invariant to that pre-defined set of permutations. For our applications we consider all permutations from $S_{n_1} \times S_{n_2} \times ... \times S_{n_{m}}$, as explained below.

\begin{figure}[H]
    \begin{minipage}{0.45\textwidth}
        \centering
        \includegraphics[width=0.9\textwidth]{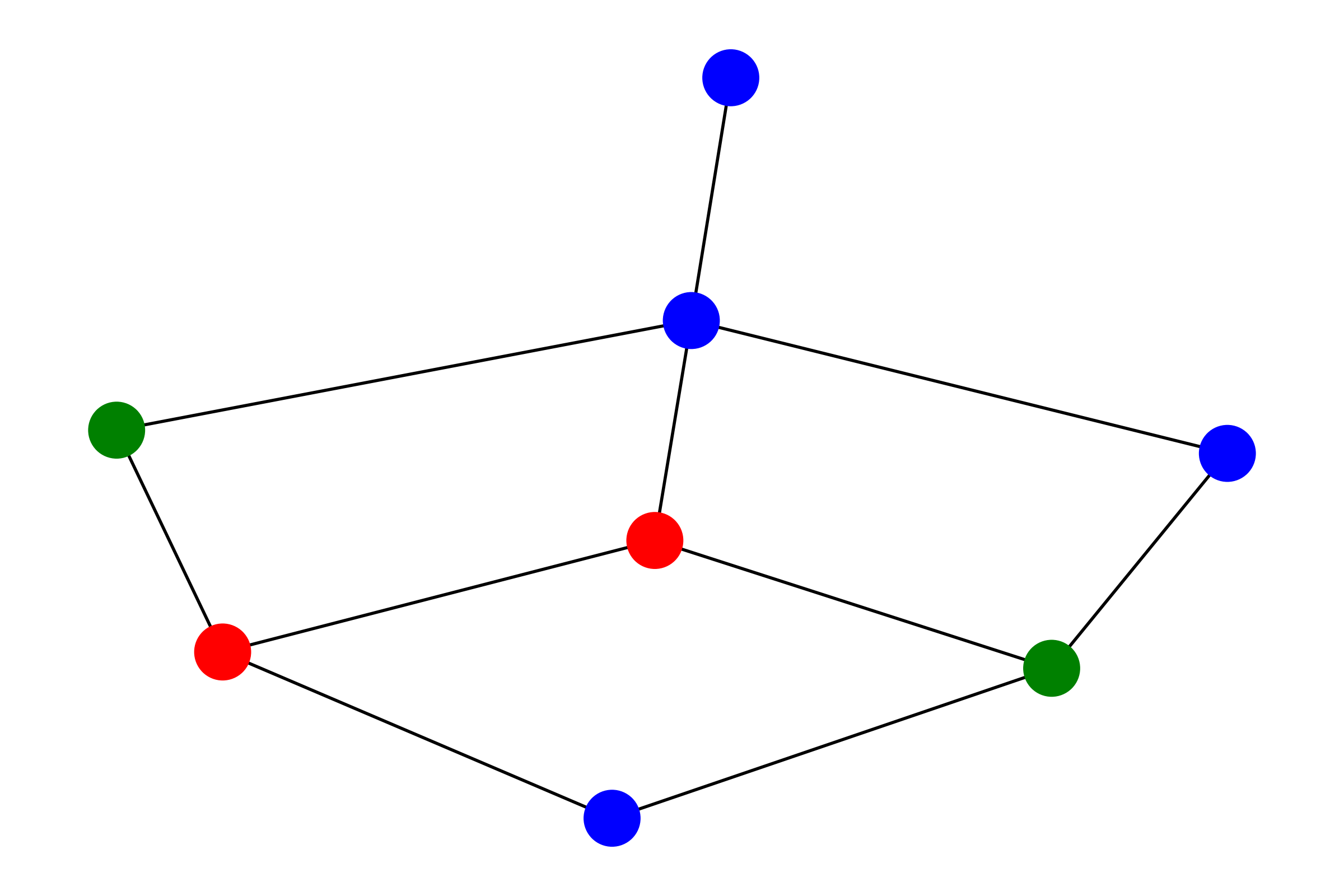}
    \end{minipage}
    \caption{Graph with nodes of three different types represented by colors. Properties shared by nodes of one specific type are hard to capture by a global aggregation.}
    \label{graphs_examples}
\end{figure}

\begin{figure}[H]
    \begin{minipage}{0.45\textwidth}
        \centering
        \includegraphics[width=0.3\textwidth]{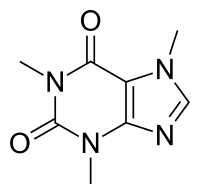}
    \end{minipage}
    \caption{A molecular graph with several types of nodes represented by different atoms and connections.}
    \label{graphs_molecule}
\end{figure}

Consider a graph with $n$ nodes of different types: $n_1$ nodes of type $1$, $n_2$ nodes of type $2$, ...., $n_m$ nodes of type $m$, $n_1 + ... + n_m = n$ (see Figures \ref{graphs_examples} and \ref{graphs_molecule}). For general graph-level tasks aggregation is performed over all $n$ nodes equivalently. Such aggregation operation guarantees invariance of the output to all $n!$ permutations from $S_n$. However, to capture properties shared by one type of node, it is viable to use aggregations only within nodes of the same type. Such aggregations would only preserve symmetries from a subgroup that permutes nodes of the same type, $S_{n_1} \times S_{n_2} \times ... \times S_{n_{m}} \subset S_{n}$.

For a special case of two node types, orthogonal bases of the new invariant linear layers are illustrated in Figures \ref{fig:tensor2} and \ref{fig:tensor3}. A map $\mathbb{R}^{n^{k}} \to \mathbb{R}^{n^{d}}$ is represented by a $(k+d)$-tensor. In particular, a map $\mathbb{R}^{n} \to \mathbb{R}^{n}$ in Figure \ref{fig:tensor2} is given by a matrix. Assume that $K_i$ is a set of nodes of type $i$ for $i = 1, 2$. Then the six maps in the Figure \ref{fig:tensor2} $L:\mathbb{R}^{n^{2}} \to \mathbb{R}$  are
$$L(x) = \sum_{i \in K_1} x_{i, i},$$
$$L(x) = \sum_{i \in K_2} x_{i, i},$$
$$L(x) = \sum_{i \neq j \in K_1} x_{i, j},$$
$$L(x) = \sum_{i \neq j \in K_2} x_{i, j},$$
$$L(x) = \sum_{i \in K_1, j \in K_2} x_{i, j},$$
$$L(x) = \sum_{i \in K_2, j \in K_1} x_{i, j}.$$
Note that relations such as ``number of edges between nodes of type $1$ and $2$'' are easier to capture with this set of maps.
\begin{figure}[H]
    \centering
    \includegraphics[width=0.45\textwidth]{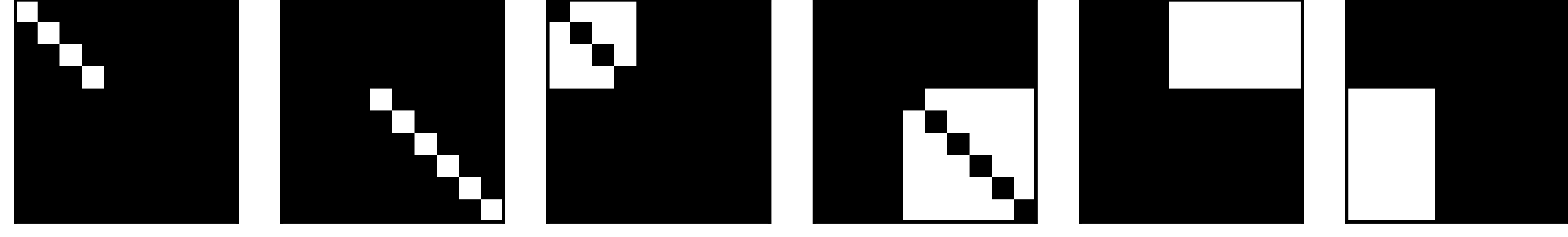}
    \caption{All six equivariant layers $\mathbb{R}^{n} \to \mathbb{R}^{n}$ for graphs with two types of nodes. White-colored entries are equal to $1$.}
    \label{fig:tensor2}
\end{figure}

\begin{figure}[H]
    \centering
    \includegraphics[width=0.45\textwidth]{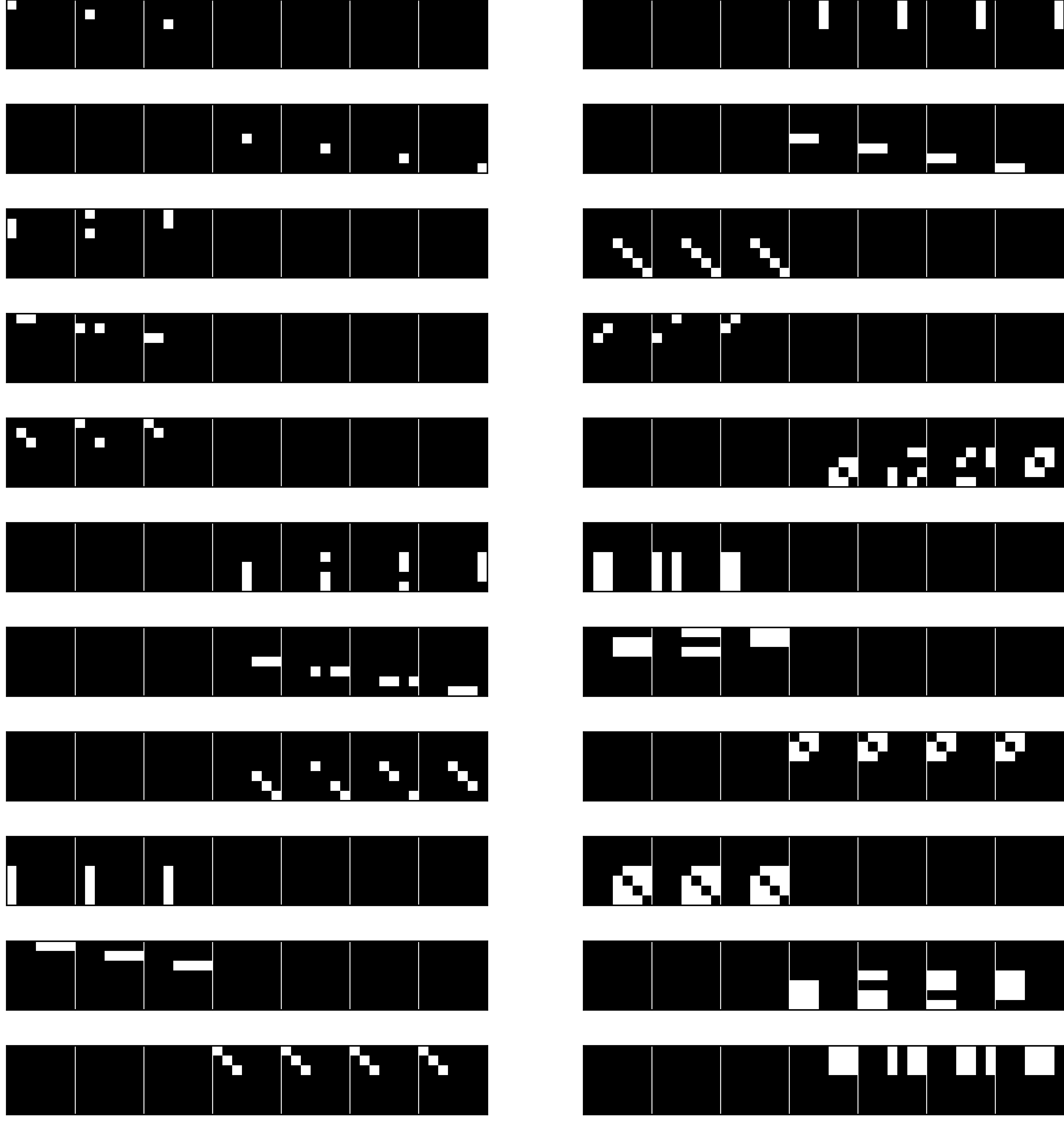}
    \caption{All 22 equivariant layers $\mathbb{R}^{n^{2}} \to \mathbb{R}^{n}$ for graphs with two types of nodes.}
    \label{fig:tensor3}
\end{figure}

%% talk about tensors and importance of tensors 
Graph data can be encoded using tensors. Node features are 1-tensors, while edge features are 2-tensors. Higher-order tensors correspond to hypergraph data such as hyper-edges \cite{Maron2018}, or non-trivial structures such as bags of rooted sub-graphs \cite{frasca2022}. Symmetries preserved by those tensors are related to the symmetries of the underlying graph.

Let $G$ be a graph with $n$ vertices, and $x$ be a feature vector $x = (x_1, ..., x_n) \in \mathbb{R}^{n}$. Where $x_i$ is a scalar feature of node $i$. If $f$ is a graph function,
$$f: \mathbb{R}^{n} \to \mathbb{R},$$
then by definition, $f$ must respect graph symmetries, i.e.
$$f(x) = f(\sigma(x)), \qquad \forall \sigma \in  \mathrm{Aut}\; G.$$
Here $\mathrm{Aut}\; G$ is a symmetry (automorphism) group of $G$. It is defined as a subgroup of all node permutations $S_n$ that preserve the graph structure
$$\mathrm{Aut}\; G = \{\sigma \in S_n | (i,j) \in E \Longleftrightarrow (\sigma(i), \sigma(j)) \in E \}.$$

For example, for the graph in Figure \ref{fig:sym}, any graph function $f(x_1, x_2, x_3, x_4)$ must be invariant to all permutations of variables $x_1, x_2, x_3$. The reason is that these permutations produce the same graph structure, and hence any function $f$ that depends only on graph structure must give the same output after such permutation.

For a general subgroup of $S_n$, the function approximation properties of tensor layers are discussed in \cite{Maron2019}. However, the preservation of graph structure puts a significant restriction on the subgroup $\mathrm{Aut}\; G \subset S_n$. This restriction can lead to a bound that is significantly smaller than $n(n-1)/2$.

\begin{figure}[H]
    \centering
    \includegraphics[width=0.45\textwidth]{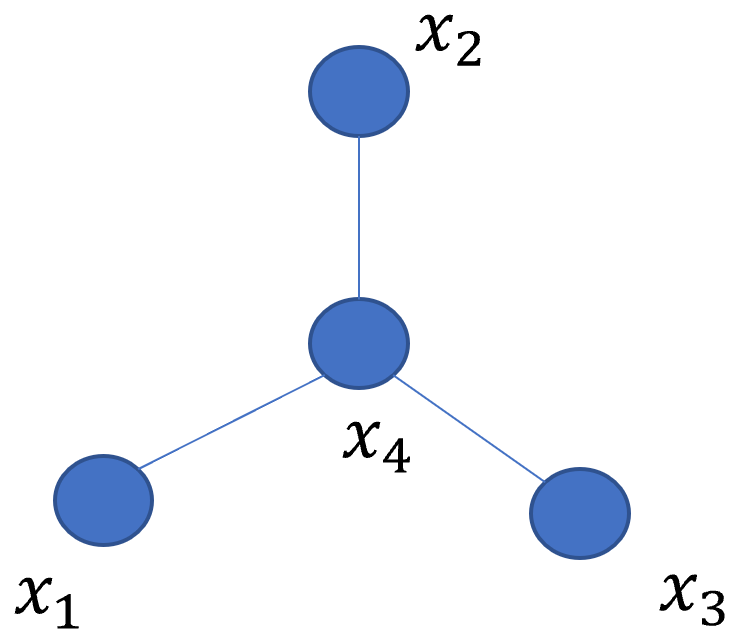}
    \caption{Graph with symmetry group of order $6$ that consists of all $3!$ permutations of nodes with with features $x_1, x_2, x_3$}
    \label{fig:sym}
\end{figure}

For a given graph $G$ with $n$ nodes define an integer $T(G)$ as the smallest size of tensors that allows invariant neural network to approximate any $\mathrm{Aut}\; G$-invariant function $f: \mathbb{R}^{n} \to \mathbb{R}$. That is, for any continuous $\mathrm{Aut}\; G$-invariant function $f$, any compact set $K \subset \mathbb{R}^{n}$ and any $\epsilon > 0$ there should exist an invariant neural network $F$ with tensor sizes $\leq T(G)$ such that
$$\max_{x \in K} |F(x) - f(x)| < \epsilon.$$

\begin{center}
\textbf{Open Question:}
Can the bound $T(G) \leq n(n - 1)/2$ be improved?
\end{center}

We note that the work \cite{Maron2019} connected the quantity $T(G)$ to degrees of polynomial generators of ring of invariants $\mathbb{R}[x_1 ,..., x_n]^{\mathrm{Aut}\; G}$.

\section{Classification of Invariant Layers}\label{classification}
\label{approach}
Recall that a linear map $f: \mathbb{R}^{n^{k}} \to \mathbb{R}$ can be written as 
$$f(x) = \sum a_{i_1, ..., i_k} x_{i_1, ..., i_k}.$$
The condition that a map $f$ is invariant to a subgroup $G \subset S_n$ of permutations is equivalent to a set of fixed points equations
$$\sum a_{i_1, ..., i_k} x_{i_1, ..., i_k} = \sum a_{i_1, ..., i_k} x_{P(i_1), ..., P(i_k)},$$
$$\forall P \in G.$$
Solving these equations can be reduced to a certain technical calculation from the branch of invariant theory and representation theory \cite{Fulton2004}. We provide this calculation in appendix A for permutations preserving node types (Theorem \ref{labeldim}), cyclic shifts (Theorem \ref{cyclic}), and translations (Theorem \ref{cnn_tensor}).

\begin{theorem}\label{labeldim}
If a graph contains nodes of $m$ different types, then the dimension of space of invariant layers $\mathbb{R}^{n^k} \to \mathbb{R}$ is given by the coefficient in front of $x^{k}/k!$ in the expression
$$e^{m(e^{x} - 1)} =$$
$$1 + m\frac{x}{1!} + (m^{2} + m)\frac{x^{2}}{2!} + (m^{3} + 3m^{2} + m)\frac{x^{3}}{3!} + ...$$
The dimension of space of equivariant layers $\mathbb{R}^{n^k} \to \mathbb{R}^{n^d}$ is given by the coefficient in front of $x^{k + d} / (k + d)!$.
\end{theorem}

\begin{table}[H]
    \centering
    \begin{tabular}{|c|c|}
        \hline
        Space & Dimension of Invariant Subspace \\ \hline
        1-tensors & $m$ \\ \hline
        2-tensors & $m^{2} + m$ \\ \hline
        3-tensors & $m^{3} + 3m^{2} + m$ \\ \hline
        4-tensors & $m^{4} + 6m^{3} + 7m^{2} + m$ \\ \hline
    \end{tabular}
    \caption{The obtained sequence of dimensions of invariant subspace for graphs with $m$ node types can be viewed as a generalization of Bell numbers.}
    \label{tab:tensors}
\end{table}

We provide an explicit basis for classified invariant layers.
\begin{theorem}\label{explicitform}
An orthogonal basis in space of $S_{n_1} \times  ... \times S_{n_{m}}$-invariant tensor layers $\mathbb{R}^{n^{k}} \to \mathbb{R}$ is given by the following set. For every disjoint partition
$$\bigsqcup_{j=1}^{m} T_j = \{1, ..., k\},$$
and for every tuple $(e_{\gamma_1}, ..., e_{\gamma_{m}})$, where each $B_{\gamma_j}$ is a basis vector in space of $S_{|T_{j}|}$-invariant layers $\mathbb{R}^{n^{|T_j|}} \to \mathbb{R}$, as in Theorem 1 from \cite{Maron2018}, form a vector
$$B_{\gamma_{1}, ..., \gamma_{m}}$$
by setting the coefficient in front of $e_{i_1, i_2, ..., i_k}$ to $1$ if and only if the coefficient in front of $e_{(i_s, s \in T_{j})}$ in $B_{\gamma_j}$ is $1$ for all $j$. Equivalently, $B_{\gamma_{1}, ..., \gamma_{m}}$ is a tensor product of $B_{\gamma_j}$ when put at appropriate indices.
\end{theorem}

%\begin{theorem}\label{expressive}
%Any pair of graphs that can be distinguished by Message Passing Neural Network with one-hot labels %on nodes of different types also can be distinguished by an Invariant Neural Network with tensor %layers $L_i$ from Theorem \ref{explicitform}.
%\end{theorem}

To illustrate the generality of the methods we develop, we also provide the results for cyclic permutations and translations. The proofs of the following theorems can be found in Appendix A.
\begin{theorem}\label{cyclic}
The dimension of space of cyclically invariant maps $\mathbb{R}^{n^{k}} \to \mathbb{R}$ is equal to $n^{k - 1}$.
\end{theorem}
\begin{theorem}\label{cnn_tensor}
The dimension of space of translation-invariant maps $\mathbb{R}^{(n^{2})^{k}} \to \mathbb{R}$ is equal to $n^{2k - 2}$. Here $\mathbb{R}^{n^{2}}$ is the space of $n\times n$ images with the action of translation group $C_n \times C_n$.
\end{theorem}

\section{Tensor Sizes for Function Approximation on Graphs}\label{maximal_tensor}
In the following section, we discuss an open question about function approximation on graphs. The question was raised in \cite{Maron2019} and \cite{keriven2019}.

Given a subgroup $H \subset S_n$, there is a construction of an $H$-invariant neural network that uses tensors of size up to $n(n-1)/2$ and achieves universal approximation property for $H$-invariant functions, see Theorem 3 in \cite{Maron2019}. When implementing a layer with tensors of size $n(n-1)/2$, the number of neurons reaches $n^{n(n-1)/2}$. Such scale is impractical, and further optimization of tensor sizes is needed.

Recall that an integer $T(G)$ is defined as the smallest bound on tensor sizes that allow continuous function approximation on graphs using Invariant Neural Network. We propose the following conjectures.
\begin{conjecture}[A]\label{conj_a}
For a graph $G$ with $n$ nodes, we have $T(G) \leq n$.
\end{conjecture}
\begin{conjecture}[B]\label{conj_b}
The value $T(G)$ is upper bounded by the maximal size of $\mathrm{Aut}\; G$ orbit.
\end{conjecture}
We experimentally verified that these conjectures hold for all graphs with $n \leq 7$.

We illustrate the relevance of these conjectures with a well-known application - universal approximation of translation-invariant functions with convolutional neural networks \cite{zhou2018}. Image classes are often assumed to be translation-invariant functions. If the input data is given by $d \times d$ images, then the translation group is a product of cyclic groups $C_d \times C_d$ (horizontal and vertical shifts). Note that we ignore rotations and reflections for simplicity.
\begin{theorem}\label{cnn}
Invariant neural networks with tensor layers of size up to $2d - 1$ can approximate any translation-invariant continuous function on $d \times d$ image data.
\end{theorem}
We provide the proof of the Theorem \ref{cnn} below. The first key step in the proof is a change of basis that diagonalizes the action of the commutative group $C_d \times C_d$. The second key step is the connection of invariance to the notion of zero-sum sequences and the Davenport constant of a group, an idea that seems new in the machine learning literature.
\begin{proof}[Proof of Theorem \ref{cnn}]
The group of translations $C_d \times C_d$ acting on images can be viewed as a subgroup of $S_{d^{2}}$. By Theorem 1 from \cite{Maron2019}, an invariant neural network can approximate any function invariant to $C_d \times C_d$. The same work shows an upper bound $d^2(d^2 - 1)/2$ on tensor sizes required for $C_d \times C_d$-invariant function approximation. Let us show that only tensors of size up to $2d - 1$ are required. From the proof of Theorem 1 in \cite{Maron2019} we know that it suffices to approximate generators in the ring of $C_d \times C_d$ invariant polynomials. The ring in question has $d^2$ variables $\mathbb{R} [ x_{1,1}, x_{1,2}, ..., x_{d, d - 1}, x_{d, d} ]$, and the action of $C_d \times C_d$ performs cyclic shifts on first and second indices of variables. The invariants of this action are not easy to analyze in basis $x_{i, j}$. We propose a change of basis that diagonalizes this action.  

Define a new basis $z_{a, b}$ as follows
$$z_{a,b} = \sum_{i=1}^{d}\sum_{j=1}^{d} x_{i, j}(\sqrt[d]{-1})^{a(i - 1) + b(j - 1)}.$$
Then the action of the translation $(p, q) \in C_d \times C_d$ on $z_{a, b}$ is multiplication by  $\sqrt[d]{-1}^{pa + qb}$, i.e. multiplication by a root of unity.

If a polynomial in variables $z_{a, b}$ is invariant to the action of $C_{d} \times C_{d}$, then every its monomial
$$\prod_{a, b = 1}^{d} z_{a, b}^{\alpha_{a, b}}$$
is also invariant, i.e. 
$$\sum_{a, b = 1}^{d} a \cdot \alpha_{a, b} = \sum_{a, b = 1}^{d} b \cdot \alpha_{a, b} = 0.$$
This relation is a zero-sum relation on a sequence of length $\sum \alpha_{a, b}$ that contains $\alpha_{a, b}$ elements $(a, b)$ from the group $C_d \times C_d$. We conclude that finding invariant tensor layers for Convolutional Neural Networks is equivalent to finding zero-sum sequences in the group $C_d \times C_d$.

Davenport constant of a group $G$ is defined as the maximal length of a sequence of elements from $G$ that contain no zero-sum subsequence. Davenport constant of the group $C_{d} \times C_{d}$ was computed before \cite{davenport} and is equal to $2d - 1$. It follows that any sequence of length $> 2d - 1$ contains a non-empty zero-sum subsequence. In terms of invariant monomials, it means that any invariant monomial of degree $2d$ and above can be written as a product of invariant monomials of degrees $\leq 2d - 1$. Hence all generators of the ring of $C_d \times C_d$-invariants have degrees $\leq 2d - 1$. By a connection established in the works \cite{yarotsky}, \cite{Maron2019}, we conclude that $2d - 1$ is an upper bound on the required tensor size in Convolutional Neural Networks.
\end{proof}

\section{Experiments}
\label{experiments}

To support the claim that our invariant/equivariant layer design improves learning on graphs with different node types, we consider open benchmarks with tasks that require learning interactions between groups of nodes, such as subgraph tasks. We compare our models to three recent architectures that achieved state-of-the-art or close results: SubGNN \cite{alsentzer2020}, GLASS \cite{glass}, GNN-Seg (treating a single group of nodes while ignoring the rest of the graph).

The training process in all experiments uses Adam optimizer \cite{adam} and ReduceLROnPlateau learning rate scheduler. The number of iterations in training is bounded by 10000, and early stopping is performed based on a non-increase of the validation data score for 1000 iterations. The models were implemented with Pytorch \cite{fey}.

\subsection{Real datasets}
We evaluate the model performance on four real-world datasets with two node types: ppi-bp, em-user, hpo-metab, hpo-neuro, with 80:10:10 training, validation, and test split.

We use the node labeling trick with Message Passing Neural Network architecture similar to GLASS. We add two more layers: $S_{n_1} \times S_{n_2}$-equivariant layer $\mathbb{R}^{n} \to \mathbb{R}^{n}$ and $S_{n_1} \times S_{n_2}$-invariant $\mathbb{R}^{n} \to \mathbb{R}$ layer instead of sum or average pooling.

The proposed model achieves close to state-of-the-art results on 3 out of 4 datasets. We note that the model variance is noticeably higher. One possible explanation is the high sensitivity caused by added learnable mappings.

\begin{table}[H]
    \centering
    \begin{tabular}{|c|c|c|c|c|}
        \hline
        Task & ppi-bp & hpo-metab  \\ \hline
        GLASS & $\mathbf{0.619 \pm 0.007}$ & $\mathbf{0.614\pm 0.005}$  \\ \hline
        SubGNN & $0.599\pm 0.008$ & $0.537\pm 0.008$ \\ \hline
        GNN-Seg & $0.361 \pm 0.008$ & $0.542\pm 0.009$  \\ \hline
        \textbf{This work} & $\mathbf{0.625 \pm 0.017}$ & $\mathbf{0.611 \pm 0.024}$ \\ \hline
    \end{tabular}
    
    \vspace{0.3cm}
    
    \begin{tabular}{|c|c|c|c|c|}
        \hline
        Task  & hpo-neuro & em-user \\ \hline
        GLASS & $\mathbf{0.685\pm 0.005}$ & $0.888\pm 0.006$ \\ \hline
        SubGNN & $0.644 \pm 0.006$ & $0.816\pm 0.13$ \\ \hline
        GNN-Seg & $0.647 \pm 0.001$ & $0.725 \pm 0.003$ \\ \hline
        \textbf{This work} & $0.661 \pm 0.009$ & $\mathbf{0.901 \pm 0.029}$ \\ \hline
    \end{tabular}
    \caption{Mean and standard deviation of Micro-F1 score over 10 runs. Baselines are taken from \cite{glass}.}
    \label{tab:synthetic}
\end{table}

\subsection{Synthetic datasets}
We use four synthetic datasets introduced in \cite{alsentzer2020}: density, cut ratio, coreness, and component. We follow the 50:25:25 training, validation, and test split as in \cite{alsentzer2020}. Our model for synthetic data follows an invariant neural network architecture. In particular, we use three equivariant $\mathbb{R}^{n} \to \mathbb{R}^{n}$ layers, followed by an invariant $\mathbb{R}^{n} \to \mathbb{R}$ pool layer and Muli-Layer-Perceptron. The vector-form implementation of the used layers is given in Appendix \ref{appC}.

We compare the performance of the state-of-the-art models and our model on these tasks, see table \ref{tab:synthetic}. The proposed model achieves state-of-the-art or similar results on 4 out of 4 synthetic datasets.

\begin{table}[H]
    \centering
    \begin{tabular}{|c|c|c|c|c|}
        \hline
        Task & density & cut ratio  \\ \hline
        GLASS & $0.930\pm 0.009$ & $\mathbf{0.935 \pm 0.006}$ \\ \hline
        SubGNN & $0.919\pm 0.006$ & $0.629\pm 0.013$  \\ \hline
        GNN-Seg & $\mathbf{0.952 \pm 0.006}$ & $0.346 \pm 0.011$ \\ \hline
        \textbf{This work} & $\mathbf{0.949 \pm 0.008}$ & $\mathbf{0.947\pm 0.008}$ \\ \hline
    \end{tabular}

    \vspace{0.3cm}
    
    \begin{tabular}{|c|c|c|c|c|}
        \hline
        Task  & coreness & component \\ \hline
        GLASS & $0.840\pm 0.009$ & $\mathbf{1.000 \pm 0.000}$ \\ \hline
        SubGNN  &  $0.659 \pm 0.031$ & $0.958 \pm 0.032$ \\ \hline
        GNN-Seg &  $0.593 \pm 0.012$ & $\mathbf{1.000 \pm 0.000}$ \\ \hline
        \textbf{This work} & $\mathbf{0.847 \pm 0.011}$ & $\mathbf{1.000\pm 0.000}$ \\ \hline
    \end{tabular}
    
    \caption{Mean and standard deviation of Micro-F1 score over 10 runs. Baselines are taken from \cite{glass}.}
    \label{tab:synthetic}
\end{table}
 %Baseline results are taken from \cite{maron_nips}. 

\section{Conclusion}
In this work, we presented a complete classification of linear tensor layers invariant to permutations of nodes of the same type. We experimentally verified the performance improvement these layers show on real and synthetic tasks. New steps have been made to further bound the size of tensors required for function approximation on graph data. In particular, when treating image data as graph data, we obtained tight bounds on the sizes of invariant convolutional tensor layers.

\section{Acknowledgements}
The work of Z.-Q. Luo was supported in part by the National Key Research and Development Project under grant 2022YFA1003900, and in part by the Guangdong Provincial Key Laboratory of Big Data Computing.

\nocite{langley00}

\bibliography{example_paper}
\bibliographystyle{icml2022}

%%%%%%%%%%%%%%%%%%%%%%%%%%%%%%%%%%%%%%%%%%%%%%%%%%%%%%%%%%%%%%%%%%%%%%%%%%%%%%%
%%%%%%%%%%%%%%%%%%%%%%%%%%%%%%%%%%%%%%%%%%%%%%%%%%%%%%%%%%%%%%%%%%%%%%%%%%%%%%%
% APPENDIX
%%%%%%%%%%%%%%%%%%%%%%%%%%%%%%%%%%%%%%%%%%%%%%%%%%%%%%%%%%%%%%%%%%%%%%%%%%%%%%%
%%%%%%%%%%%%%%%%%%%%%%%%%%%%%%%%%%%%%%%%%%%%%%%%%%%%%%%%%%%%%%%%%%%%%%%%%%%%%%%
\newpage
\appendix
\onecolumn

\section{Proofs of Main Theorems}\label{appA}
\begin{proof}[Proof of Theorem \ref{labeldim}]
Decompose the space $\mathbb{R}^{n}$ into a direct sum of subspaces where permutations $S_{n_1} \times ... \times S_{n_m}$ act,
$$\mathbb{R}^{n} = U_1 \oplus U_2 \oplus ... \oplus U_m.$$
Rewrite the tensor product
$$(U_1 \oplus U_2 \oplus ... \oplus U_m)^{\otimes k}$$
into multinomial sum, see \cite{Fulton2004},
$$\bigoplus\limits_{k_1, ..., k_m} \binom{k}{k_1, ..., k_m} U_{1}^{\otimes k_1} \otimes ... \otimes U_{m}^{\otimes k_m}.$$
For example,
$$(U_1 \oplus U_2)^{\otimes 2} = (U_1 \otimes U_1) \oplus (U_1 \otimes U_2) \oplus (U_2 \otimes U_1) \oplus (U_2 \otimes U_2).$$

Using the result of \cite{Maron2018}, we note that the dimension of $S_{n_{j}}$ invariants of $U_{j}^{\otimes k_{j}}$ is equal to $B(k_{j})$. Hence the dimension of $S_{n_1} \times ... \times S_{n_m}$ invariants is given by the sum
$$\sum\limits_{k_1 + ... + k_m = k} \binom{k}{k_1, ..., k_m} B(k_1)B(k_2) ... B(k_m).$$
The expression above is known in the theory of exponential generating functions \cite{Stanley}, see Lemma \ref{suppl1}. we conclude that this sum appears as a coefficient in front of $x^{k}/k!$ in the series
$$\left( \sum_{i=0}^{\infty} B(k)\frac{x^{k}}{k!}\right)^{m} = \left(e^{e^{x} - 1} \right)^{m} = e^{m(e^{x} - 1)}.$$
For the claim about equivariant maps $\mathbb{R}^{n^{k}} \to \mathbb{R}^{n^{d}}$ see Lemma \ref{suppl2}.
\end{proof}

\begin{proof}[Proof of Theorem \ref{explicitform}]
Consider the set of vectors in $\mathbb{R}^{n^{k}}$ obtained by the following procedure
\begin{enumerate}
    \item For each node type $j$ select a subset of $T_j$ of indices from $\{1, ..., k\}$, in such a way that 
    $$\bigsqcup_{j=1}^{m} T_j = \{1, ..., k\}.$$
    \item For each subspace $\mathbb{R}^{n^{|T_j|}}$ we select a basis element $B_{\gamma_j}$, where $\gamma_j$ is a partition of $T_j$, according to a construction of the basis in \cite{Maron2018}. The basis element $B_{\gamma_1, ..., \gamma_m}$ is then formed by taking the tensor product of all vectors $B_{\gamma_j}$, with $B_{\gamma_j}$ located at indices $T_j$.
\end{enumerate}
On the one hand, there are
$$\sum\limits_{k_1 + ... + k_m = k} \binom{k}{k_1, ..., k_m} B(k_1)B(k_2) ... B(k_m)$$
vectors in this set. On the other hand, they are orthogonal to each other. Indeed, assume that two elements $B_{\gamma_1, ..., \gamma_m}$ and $B_{\beta_1, ..., \beta_m}$ share a common non-zero coefficient in front of some element $e_{i_1} \otimes ... \otimes e_{i_k}$. It follows that $T_{j}$ can be defined as the set of indices $t$ such that node $i_t$ has type $j$. But then, by the definition of $B_{\gamma_j}$, the element $\bigotimes_{t \in T_j} e_{i_t}$ uniquely defines an equivalence class (partition) $\gamma_j$. Hence $\gamma_j = \beta_j$ for all $j$.

\end{proof}

\begin{proof}[Proof of Theorem \ref{cyclic}]
The basis of cyclically invariant $k$-tensors can be obtained by projecting $k$-tensors of the form $e_{i_1} \otimes ... \otimes e_{i_k}$ on the invariant subspace using the averaging operator 
$$\frac{1}{n}\sum_{g \in C_n} g \cdot e_{i_1} \otimes ... \otimes e_{i_k}.$$
It follows that bases and dimensions of cyclically-invariant subspaces in $\mathbb{R}^{n^{k}}$ and over $\mathbb{C}^{n^{k}}$ are the same.

The cyclic action of $C_n$ on $\mathbb{C}^{n}$ can be diagonalized, resulting in decomposition
$$\mathbb{C}^{n} = \bigoplus_{i=0}^{n-1} V_{i},$$
where the cyclic action on $V_k$ is multiplication by $(\sqrt[n]{-1})^{k}$. Let $d$ be the dimension of the invariant subspace in
$$(\mathbb{C}^{n})^{\otimes k} = (\bigoplus_{i=0}^{n-1} V_{i})^{\otimes k}.$$
The shift $V_0 \mapsto V_1$, $V_1 \mapsto V_2$, ..., $V_{n - 1} \mapsto V_0$ does not change the  decomposition but maps the invariant subspace to subspace where cyclic action is multiplication by $\sqrt[n]{-1}$. It follows that this subspace also has dimension $d$. Repeating the argument we arrive at $d + d + ... + d = n^{k}$, hence $d = n^{k}/n = n^{k - 1}$.
\end{proof}

\begin{proof}[Proof of Theorem \ref{cnn_tensor}]
An $n \times n$ image is a $2$-tensor from $\mathbb{R}^{n} \otimes \mathbb{R}^{n}$. Vertical translations act on the first entry of a $2$-tensor while horizontal translations act on the second. The maps $(\mathbb{R}^{n} \otimes \mathbb{R}^{n})^{\otimes k}$ invariant to translations are then computed as
$$((\mathbb{R}^{n} \otimes \mathbb{R}^{n})^{\otimes k})^{C_n \times C_n} = (\mathbb{R}^{n^{k}})^{C_n} \otimes (\mathbb{R}^{n^{k}})^{C_n}.$$
Since cyclic invariants are computed in Theorem \ref{cyclic}, the dimension of last tensor product is equal to $n^{k - 1} \cdot n^{k - 1} = n^{2k - 2}$.
\end{proof}

\section{Supplementary Lemmas}\label{appB}
\begin{lemma}\label{suppl1}
Let $(a_n)$ be a sequence and let $f(x)$ be the exponential generating function of that sequence.
$$f(x) = \sum_{n = 0}^{\infty} a_n \frac{x^{n}}{n!}.$$
Then $f(x)^{m}$ is an exponential generating function for the sequence
$$b_k = \sum\limits_{k_1 + ... + k_m = k} \binom{k}{k_1, ..., k_m} a_{k_1} a_{k_2} ... a_{k_m}.$$
\end{lemma}
\begin{proof}
We start by expanding $f(x)^m$:

$$f(x)^m = \left(\sum_{n = 0}^{\infty} a_n \frac{x^n}{n!}\right)^m =$$
$$= \sum_{k_1 = 0}^{\infty} a_{k_1} \frac{x^{k_1}}{k_1!} \sum_{k_2 = 0}^{\infty} a_{k_2} \frac{x^{k_2}}{k_2!} \cdots \sum_{k_m = 0}^{\infty} a_{k_m} \frac{x^{k_m}}{k_m!} =$$
$$= \sum_{k_1 = 0}^{\infty} \sum_{k_2 = 0}^{\infty} \cdots \sum_{k_m = 0}^{\infty} a_{k_1} a_{k_2} \cdots a_{k_m} \frac{x^{k_1 + k_2 + \cdots + k_m}}{k_1!k_2!\cdots k_m!} =$$
$$= \sum_{k = 0}^{\infty} \left(\sum_{k_1 + k_2 + \cdots + k_m = k} a_{k_1} a_{k_2} \cdots a_{k_m} \frac{k!}{k_1!k_2!\cdots k_m!}\right) \frac{x^k}{k!} =$$
$$= \sum_{k = 0}^{\infty} b_k \frac{x^k}{k!}.$$
where the last step follows from the definition of $b_k$. This shows that $f(x)^m$ is an exponential generating function for the sequence $(b_k)$. Thus, we have proved Lemma \ref{suppl1}.
\end{proof}

\begin{lemma}\label{suppl2}
The dimension of space of $S_{n_1} \times ... \times S_{n_m}$-equivariant maps $\mathbb{R}^{n^{k}} \to \mathbb{R}^{n^{d}}$ depends only on $k + d$.
\end{lemma}

\begin{proof}
From the point of view of tensor algebra, the computation of $G$-equivariant layers can be viewed as the computation of $G$-equivariant linear maps
$$\mathrm{Hom}_{G}(V^{\otimes k}, V^{\otimes d}).$$

Representation theory of symmetric group $S_n$ is well-studied. In particular, it is known \cite{Fulton2004} that all characters of $S_{n_1} \times ... \times S_{n_m}$ are real-valued. Hence action on the dual space $V^{*\otimes d}$ is equivalent to the action on the original space $V^{\otimes d}$. Hence
$$\mathrm{Hom}_{S_{n}}(V^{\otimes k}, V^{\otimes d}) = (V^{\otimes k} \otimes V^{*\otimes d})^{S_n} = (V^{\otimes k + d})^{S_n}.$$
This shows that the answer can depend only on $k + d$. 
\end{proof}

\section{Implementation}\label{appC}
Let $K_1, K_2, ..., K_m$ be sets of nodes from groups $1$ to $m$,
$$\bigsqcup_{i=1}^{m} K_i = \{1, 2, ..., n\}.$$
Denote by $1_{K_1}, ..., 1_{K_m}$ the $n$-dimensional vectors with $1_{K_i}$ having coordinate $1$ only at indices $K_i$. And let $I_{K_i}$ be an identity matrix with ones only at indices from $K_i$. 

An $S_{n_1} \times ..., \times S_{n_m}$-invariant layer $L: \mathbb{R}^{n} \to \mathbb{R}$ has a form
$$L(x) = \sum_{i=1}^{m} w_i 1_{K_i}^{t} x,$$
where $w_i \in \mathbb{R}$ are learnable parameters.

An $S_{n_1} \times ..., \times S_{n_m}$-equivariant layer $L: \mathbb{R}^{n} \to \mathbb{R}^{n}$ has a form
$$L(x) = \sum_{i, j = 1}^{m} w_{i, j} (1_{K_i}^{t} x) 1_{K_j} + \sum_{i = 1}^{m} w_{i} I_{K_i} x,$$
where $w_{i,j}$ and $w_i$ are learnable parameters.

\end{document}